\pdfoutput=1

\documentclass[11pt]{article}

\usepackage[final]{acl}

\usepackage{times}
\usepackage{latexsym}

\usepackage[T1]{fontenc}

\usepackage[utf8]{inputenc}

\usepackage{microtype}

\usepackage{inconsolata}

\usepackage{graphicx}

\usepackage{amsmath}
\usepackage{amsthm}
\usepackage{amsfonts}
\usepackage{enumerate}
\usepackage{xspace}
\usepackage{ifthen}
\usepackage{tabularx}
\usepackage{amssymb}
\usepackage{soul}

\usepackage{resizegather}
\usepackage{algorithm}
\usepackage[noend]{algpseudocode} 
\algrenewcommand{\algorithmiccomment}[1]{\textcolor{gray}{$\triangleright$ #1}}

\usepackage{booktabs}
\usepackage{multirow}
\usepackage{subcaption}

\theoremstyle{plain}
\newtheorem{theorem}{Theorem}
\newtheorem{lemma}{Lemma}
\theoremstyle{definition}
\newtheorem{definition}{Definition}

\newcommand{\todop}[1]{{\color[HTML]{F5252C} \textbf{(TODO: #1)}}}
\newcommand{\todocite}[1]{\todop{CITE\ifthenelse{\equal{#1}{}}{}{ #1}}}

\newcommand{\Displacement}{\Delta_G}

\newcommand{\LangPfx}{\textit{Pr}}
\newcommand{\SynCon}{\equiv_{L(G)}}
\newcommand{\SynConApprox}{\sim_{L(G)}}

\title{Accelerating Constrained Decoding with Token Space Compression}

\author{Michael Sullivan and Alexander Koller \\
  Department of Language Science and Technology \\
  Saarland Informatics Campus \\
  Saarland University, Saarbrücken, Germany \\
  \texttt{\{msullivan, koller\}@coli.uni-saarland.de} 
}

\begin{document}
\maketitle
\begin{abstract}
To guarantee that an LLM's outputs conform to a specified structure, context-free grammar (CFG) decoding engines force the selection of next tokens to produce strings that conform to a given CFG. Current CFG-constrained decoding engines are highly optimized, but still suffer from the inherent costs arising from their massive per-step search space\textemdash i.e.\ the entire token vocabulary. This results in intractably high overhead for more complex CFGs, which is precisely the situation where CFG engines are most useful. In this paper, we introduce \textsc{CFGzip}, an offline technique for compressing the token search space, which massively reduces CFG engine overhead. In experiments, we report latency reduction of up to 75x during batched inference, cutting overhead down to $\sim$1.2-2x on the hardest grammars: with \textsc{CFGzip}, constrained decoding is now possible at scale for complex CFGs.
\end{abstract}

\section{Introduction}
\label{sec_intro}







Many modern LLM applications such as coding \cite{huynh2025large,jiang2026survey}, tool calling \cite{qintoolllm,sullivan2025procedural,patil2025berkeley}, etc. necessitate structured generation: the generation of text that conforms to a specific format. Although LLMs excel at generating syntactically correct text in formats frequently found in their training data \cite{bogin-etal-2024-leveraging,mundler2025type}, they struggle to conform to the specifications of novel formats such as domain-specific languages and unique programming languages (see Section \ref{sec_experiments}). In such situations, it is therefore desirable to constrain the model's outputs to guarantee syntactic correctness. 

\begin{figure}[t]
\centering
\includegraphics[width=\linewidth]{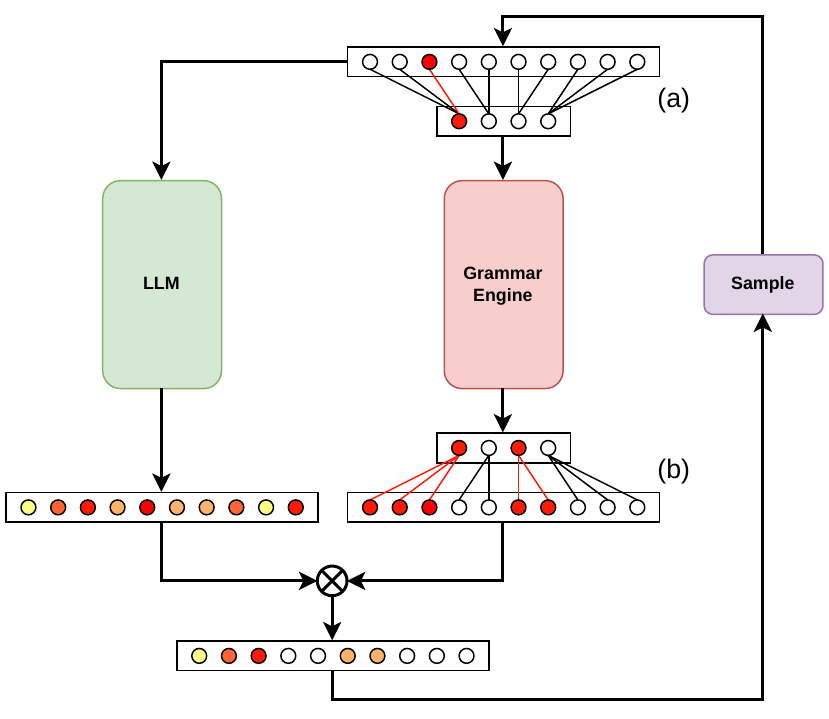}
\caption{\textsc{CFGzip} first pre-computes equivalence classes over the token vocabulary. At inference time, the grammar engine operates over the smaller vocabulary of equivalence class representatives: (a) a sampled token is mapped to the representative of its equivalence class; and (b) the grammar engine produces a mask over the equivalence class vocabulary, which is then mapped to a mask over the token vocabulary.}
\label{fig_fig1}
\end{figure}

Structural constraints for LLM decoding are often defined by context-free grammars \cite[CFGs;][etc.]{chomsky1956three,angelov-2009-incremental,dong2025xgrammar}, which are strictly more expressive than regular expressions, and encompass use cases such as the JSON file format, XML, many programming languages, chemical structure notation \citep[e.g.][]{weininger1988smiles}, etc. 

To guarantee that LLM outputs conform to a given CFG, constrained decoding engines are employed to modify the LLM's next-token distribution, masking the probability of all tokens which would lead to a generated sequence that violates the format specified by the CFG.

\begin{figure*}[t]
\begin{center}
\includegraphics[width=1.0\textwidth]{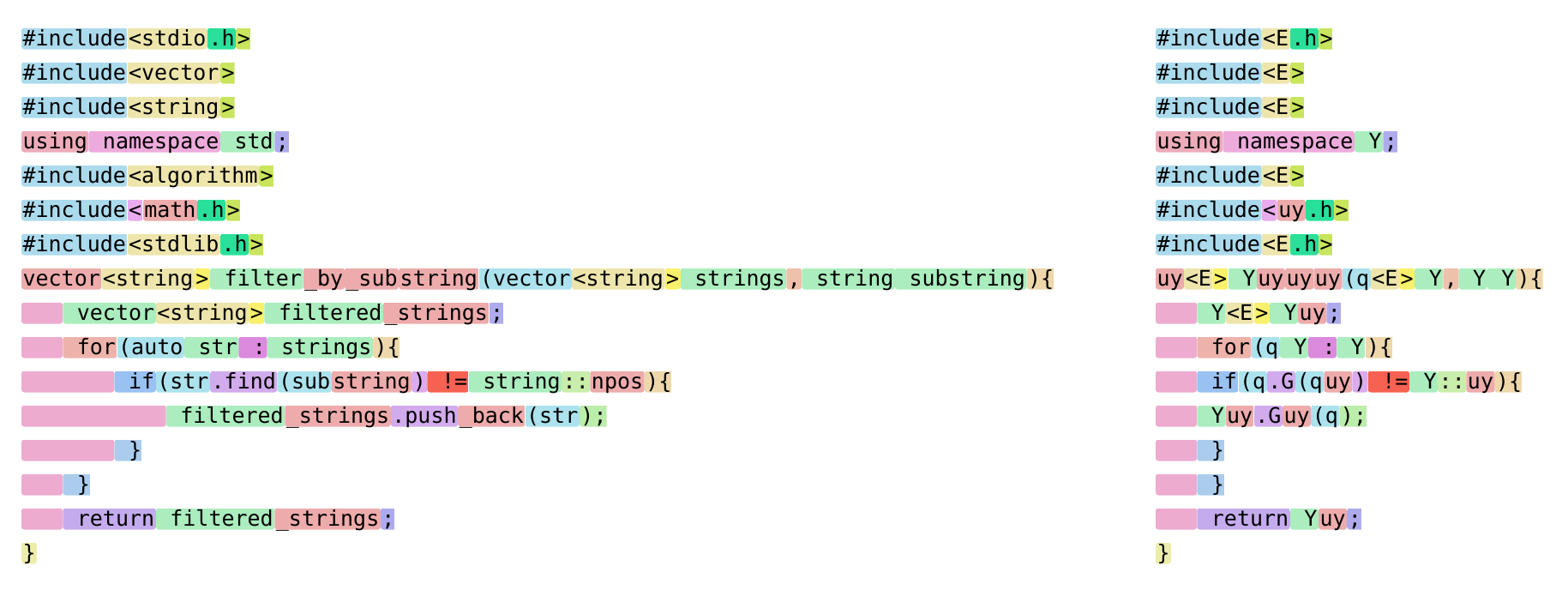}
\end{center}
\caption{C++ code generated by Llama-3.2-3B-Instruct using \textsc{CFGzip}. The LLM operates over the actual tokens (left), while the grammar engine operates over the smaller vocabulary of equivalence class representatives (right): in this figure, each unique equivalence class is highlighted with a unique color for reference.}
\label{fig_token_classes}
\end{figure*}

Given the large overhead already associated with parsing a CFG\textemdash $\mathcal{O}(n^3)$ in the worst case \cite{earley1970efficient}\textemdash the runtime cost associated with naively checking the validity of each token at each step can be prohibitively high. For this reason, SoTA grammar engines incorporate substantial online and offline optimization techniques. Despite these improvements, even the fastest grammar engines still introduce massive latency\textemdash a 2-9x slowdown during unbatched inference\textemdash when used with more complex CFGs (e.g.\ programming languages; see Section \ref{sec_experiments}).

To ameliorate this online overhead, we introduce \textsc{CFGzip} (Section \ref{sec_method}), a compression technique that vastly reduces the per-step runtime of the grammar engine. The key idea behind our approach is that many tokens are interchangeable with respect to the grammar, in that we need only check the validity of one to determine the validity of all tokens within a given equivalence class. 

Our method first pre-computes these equivalence classes using a novel algorithm that constructs a decidable refinement of the syntactic congruence relation. \textsc{CFGzip} then wraps the grammar engine at inference time, so that the engine operates over representatives of the equivalence classes, rather than the token vocabulary itself (see Figures \ref{fig_fig1}-\ref{fig_token_classes}). 

In this way, \textsc{CFGzip} synergizes with the grammar engine's built-in optimizations, resulting in a 9-45x reduction in overhead over the current SoTA grammar engine. This yields \textbf{an up to 7.5x reduction in unbatched constrained generation time}\textemdash i.e. the combined runtime of the LLM and grammar engine\textemdash for more complex CFGs. 

Furthermore, the relative latency reduction of our approach \textit{increases} with batch size, making constrained generation feasible at scale for complex CFGs: \textbf{\textsc{CFGzip} delivers speedups of up to $\sim$75x during batched inference} with constrained decoding, \textbf{reducing overhead to merely $\sim$1.2-2x on tasks where a SoTA grammar engine alone times out}. 

\textsc{CFGzip} is:

\begin{enumerate}
    \item \textbf{Offline:} compression is performed once per tokenizer/CFG pair, and the compressed vocabulary is cached to disk. At inference time, \textsc{CFGzip} itself incurs virtually no overhead.

    \item \textbf{Lossless:} \textsc{CFGzip} compression is provably lossless, meaning that the outputs of a grammar engine equipped with \textsc{CFGzip} are guaranteed to be byte-identical to those of the original engine.
\end{enumerate}

\textbf{We release \textsc{CFGzip} as a pip-installable Python package}\footnote{
\url{https://pypi.org/project/cfgzip}
}, and make all code used in the evaluation experiments in this paper publicly available on GitHub\footnote{
\url{https://github.com/coli-saar/cfgzip-experiments}
}. 

\section{Background and Preliminaries}
\label{sec_background}

In the following, we adopt the Kleene star \cite{kleene1956representation} notation: given a set $X$, we define $X^*$ as the set of all sequences of elements of $X$ (i.e. strings over $X$), including the empty sequence/string $\epsilon$. Given strings $w,z$, we denote their concatenation through juxtaposition: $wz$.

\subsection{Context-Free Grammars (CFGs)}
\label{sec_background_sub_cfgs}

A CFG $G=(V,\Sigma,R,S)$ is defined by a set $V$ of \textit{non-terminal symbols}, an \textit{alphabet} (set of characters) $\Sigma$, a \textit{start symbol} $S\in V$, and a set $R$ of productions (rewrite rules) of the form $A\rightarrow\alpha$, where $A\in V$ is a non-terminal and $\alpha\in(V\cup\Sigma)^*$ is a sequence of terminals and non-terminals. 

Each CFG $G$ defines a \textit{language} (set of strings) $L(G)\subseteq\Sigma^*$, where $w\in L(G)$ if and only if $S\Rightarrow^*w$: i.e. there is a finite sequence of rules $S\rightarrow\alpha_0,A_1\rightarrow\alpha_1,\dots,A_n\rightarrow\alpha_n\in R$ that derives $w$ from $S$ (a \textit{parse} of $w$). 
We additionally define the set $\textit{Pr}(L(G))$ of \textit{prefixes} of $L(G)$: strings $w\in\Sigma^*$ such that there is some (possibly empty) string $z\in\Sigma^*$ such that the concatenation $wz$ is in $L(G)$. 


\paragraph{Greibach Normal Form \citep[GNF;][]{greibach1965new}.} A CFG $G$ is in GNF if all productions in $G$ are of the form $S\to\epsilon$ (where $S$ is the start symbol) or $A\to a\hspace{1mm}\beta$, where $a\in\Sigma$ is a (non-$\epsilon$) terminal symbol and $\beta\in (V-\{S\})^*$ is a (possibly empty) sequence of non-terminal, non-start symbols.



It is a well-known result that for every CFG $G$, there is an equivalent CFG $G'$ in GNF\textemdash i.e. such that $L(G)=L(G')$. Each GNF CFG $G$ defines a simple pushdown automaton (PDA\footnote{\textit{Every} CFG defines a PDA: we defer an in-depth discussion of PDAs to prior work, and provide a semi-formal description here that is sufficient for the discussion at hand.}) that recognizes (parses) $L(G)$: this PDA consists of a stack $\sigma$ of non-terminal symbols of $G$, and a transition function\footnote{We ignore states in the definition of the transition function, as PDAs derived from GNF CFGs have only one state.} $\delta_G$ mapping pairs of terminals and non-terminals to sequences of non-terminals. For each terminal $a\in\Sigma$, non-terminal $A\in V$, and sequence of non-terminals $\beta\in V^*$,  $\beta\in\delta_G(a,A)$ if and only if $A\to a\hspace{1mm}\beta$ is a production in $G$.

Given a string $w$ in $\Sigma^*$, the PDA parses $w$ non-deterministically, with its stack initialized to $\sigma=(S)$, where $S$ is the start symbol of $G$. The PDA scans $w$ character-by-character: at the $i^{th}$ character of $w$, the PDA non-deterministically (i) pops the top symbol $A$ from $\sigma$; and (ii) pushes $\beta$ to $\sigma$ for $\beta\in\delta_G(w_i,A)$. The PDA accepts $w$ ($w\in L(G)$) if $\sigma$ is empty when the end of $w$ is reached\textemdash otherwise, $w$ is not a member of $L(G)$. 

\subsection{XGrammar2}
\label{sec_background_sub_xgrammar}

A CFG decoding engine ensures that an LLM generates a sequence that conforms to a given CFG $G$. Precisely, for each LLM output token $y_{i}$, the decoding engine ensures that the sequence $y_{:i+1}=y_{:i}y_i$ is a valid prefix of $L(G)$: i.e. $y_{:i+1}\in\textit{Pr}(L(G))$. This is achieved by logit masking: for all tokens $t^{(k)}$ in the token vocabulary $T$ such that $y_{:i}t^{(k)}\notin\textit{Pr}(L(G))$, the corresponding logit $h_k$ is set to $-\infty$, so that the next-token probability of $t^{(k)}$ is $0$.

We restrict the present discussion to an overview of XGrammar2 \cite{li2026xgrammar}\textemdash a SoTA grammar engine\textemdash because this is the engine that we pair with \textsc{CFGzip} in our experiments in Section \ref{sec_experiments}: although \textsc{CFGzip} is theoretically engine-agnostic, we were unable to find any other grammar engine that is capable of processing the more sophisticated CFGs used in our experiments (see the discussion in the Limitations section).

Underlying XGrammar2 is an Earley parser \cite{earley1970efficient}, which can run in time linear in the input string length $n$ for a specific class of CFGs\footnote{Non-right recursive $\textit{LR}(k)$ grammars.}, but has $\mathcal{O}(n^3)$ runtime in the worst-case scenario for arbitrary CFGs. Here, the Earley parser maintains a running parse of the generated sequence, and is used to check the compatibility of each token in the vocabulary to construct the next-token logit mask. 

If done naively, constructing the logit mask requires verifying each token in the vocabulary at each decoding step, resulting in a worst-case runtime of $\mathcal{O}(|T|\cdot n^3)$. To mitigate this cost, XGrammar2 introduces several optimizations that reduce the number of tokens that need to be checked at each step: for example, pre-computing tokens that are guaranteed to be rejected in a given parser state.

The heavily optimized XGrammar2 can quickly compute token masks for simple CFGs such as JSON schema grammars, resulting in negligible overhead. However, as we demonstrate empirically in Section \ref{sec_experiments}, more complex CFGs still present a challenge to XGrammar2: we record a $\sim$2-9x slowdown over the base LLM during unbatched inference with CFGs such as C++ and a modified variant of Python. 

\section{\textsc{CFGzip}}
\label{sec_method}

For a given CFG $G$ and token vocabulary $T$, \textsc{CFGzip} computes equivalence classes of tokens that are interchangeable with respect to $G$: at each step, the grammar engine needs only verify one token per equivalence class (see Figure \ref{fig_token_classes}).

In Section \ref{sec_method_sub_syn_con}, we provide an overview of the \textsc{CFGzip} pipeline using syntactic congruence \cite{myhill1957finite} as the underlying equivalence relation, for the sake of exposition; it follows as a consequence of Theorem \ref{thm_token_seq_equiv} that equivalence-class compression using syntactic congruence is lossless and engine-agnostic.

However, computing the syntactic congruence relation is undecidable for arbitrary context-free languages \cite{bar1961formal}. In Section \ref{sec_method_sub_approx}, we define a computable and tractable approximation of syntactic congruence. This relation is furthermore a refinement of syntactic congruence (Theorem \ref{thm_displacements}), which guarantees that it maintains the lossless and engine-agnostic properties of the latter.

\subsection{\textsc{CFGzip} Pipeline}
\label{sec_method_sub_syn_con}

\paragraph{Syntactic Congruence.} Intuitively, strings $t,u\in\Sigma^*$ are \textit{syntactically congruent} ($t\SynCon u$) when they are interchangeable relative to $L(G)$: each instance of $t$ in a string in $L(G)$ can be swapped with $u$ to yield a string that is also in $L(G)$.

\begin{definition}[Syntactic Congruence; \citeauthor{myhill1957finite}, \citeyear{myhill1957finite}]
\label{def_syn_congruence}
    For $t,u\in\Sigma^*$, $t\SynCon u$ if and only if for all $w,z\in\Sigma^*$: $wtz\in L(G)\leftrightarrow wuz\in L(G)$ 
\end{definition}

Given two length-$n$ sequences of tokens $\tau=t_0\dots t_{n-1}$, $\tau'=t_0'\dots t_{n-1}'$ that are element-wise congruent\textemdash i.e. $t_i\SynCon t'_i$ for all $0\leq i<n$\textemdash it is straightforward to show that $\tau$ and $\tau'$ are themselves congruent: $\tau\SynCon\tau'$ (proof in Lemma \ref{lem_syn_con_cat} of Appendix \ref{app_proof_thm_token_seq_equiv}). It then follows from Definition \ref{def_syn_congruence} that $\tau$ is a valid prefix of $L(G)$ if and only if $\tau'$ is as well.







\begin{theorem}
\label{thm_token_seq_equiv}
For any $n\in\mathbb{N}$, alphabet $\Sigma$, language $L$ over $\Sigma$, and length-$n$ sequences of tokens $\tau=t_0\dots t_{n-1}$, $\tau'=t_0'\dots t_{n-1}'$ over $\Sigma$ with $t_i\equiv_Lt_i'$ (for all $0\leq i<n$): $\tau\in\LangPfx(L)\leftrightarrow\tau'\in\LangPfx(L)$
\end{theorem}
\begin{proof}
Appendix \ref{app_proof_thm_token_seq_equiv}.
\end{proof}

A critical corollary of Theorem \ref{thm_token_seq_equiv} is that we can losslessly map between the token vocabulary $T$ that is used by the LLM, and the set of equivalence class representatives $E\subseteq T$ used by the grammar engine under \textsc{CFGzip}: replacing each token $t_i$ in a sequence with another token that is congruent to $t_i$ does not affect the validity of the sequence with respect to $G$ (see Figure \ref{fig_token_classes}).

\paragraph{Pre-Computation.} \textsc{CFGzip} first groups the token vocabulary $T$ into equivalence classes, using (an approximation of) syntactic congruence as the equivalence relation. We then assemble the set $E$ of equivalence class representatives: for each class $e$, we choose as its representative the byte-wise shortest token in $e$, as the grammar engine's underlying parser still operates on the character/byte level.


We then compute and cache a vector $c\in\mathbb{N}^{|T|}$ that maps token IDs to class IDs: for the $i^{th}$ token $t^{(i)}\in T$, $c_i$ is the ID of that token's equivalence class. We additionally compute and cache a vector $r\in\mathbb{N}^{|E|}$ that maps class IDs to representatives, such that the token $r_k$ is the representative of the $k^{th}$ equivalence class. For a token vocabulary size of 100,000 to 200,000 tokens, these two vectors occupy less than a megabyte in memory.

\paragraph{Inference.} \textsc{CFGzip} initializes the grammar engine with a vocabulary consisting of the set $E$ of equivalence class representatives, in place of the entire token vocabulary $T$ (see Figure \ref{fig_token_classes}).


At inference time, a token $t^{(i)}\in T$ is sampled from the next-token distribution, and appended to the current generation sequence. Rather than $t^{(i)}$, we pass to the grammar engine the class representative $r_{c_i}$ to update its internal state (see Figure \ref{fig_fig1}): by Theorem \ref{thm_token_seq_equiv}, this is equivalent to passing the original token $t^{(i)}$.

The LLM then produces next-token log probabilities $h\in\mathbb{R}^{|T|}$ over the token vocabulary $T$, and the grammar engine produces a binary mask $m\in\{-\infty,0\}^{|E|}$ over the equivalence class vocabulary $E$. Rather than directly pointwise adding $h$ to $m$ (which is impossible), we instead employ a gather operation so that we add to $h_i$ the mask entry whose index corresponds to the equivalence class of $t^{(i)}$: $h_i\gets h_i+m_{c_i}$. This gather-then-mask operation can be readily parallelized on a GPU, and so incurs minimal runtime overhead.

\subsection{Approximating Congruence}
\label{sec_method_sub_approx}

As syntactic congruence $t\SynCon u$ is undecidable for context-free languages, we use in practice a computable refinement $t\SynConApprox u$ for equivalence-class-based compression in the \textsc{CFGzip} pipeline of Section \ref{sec_method_sub_syn_con}.

To that end, we first convert $G$ to Greibach Normal Form (GNF), to obtain the simplified PDA discussed in Section \ref{sec_background_sub_cfgs}. Conversion to GNF results in a worst-case $\mathcal{O}(n^4)$ increase in the number of grammar productions. However, as the relation $\SynCon$ is defined relative to the \textit{language} $L(G)$ of $G$\textemdash and the approximation $\SynConApprox$ is a refinement of $\SynCon$\textemdash we need only convert $G$ to GNF for the purpose of equivalence class computation: \textsc{CFGzip} does not necessitate any modification to the grammar passed to the grammar engine itself.

\subsubsection{Displacement Equivalence}
\label{sec_method_sub_approx_sub_disp}


For each token $t\in T$, we define the \textit{displacement} $\Displacement(t)$ as the set of all pairs $(\sigma^{(i)},\sigma^{(o)})\in V^*\times V^*$ such that $\sigma^{(o)}$ is a stack that can be obtained by consuming the stack $\sigma^{(i)}$ to process $t$, using the PDA transition function $\delta_G$ (see Section \ref{sec_background_sub_cfgs}).

For any tokens $t,u\in T$, $\Displacement(t)=\Displacement(u)$ implies $t\SynCon u$: if the sets of input stacks that accept $t$ and $u$ under $\delta_G$ are identical, the sets of strings to which $t$ and $u$ can be appended to form valid prefix strings are identical. Similarly, if $t$ and $u$ each map the same input stack $\sigma^{(i)}$ to the same output stacks $\sigma^{(o)}$ under $\delta_G$, the sets of strings that can follow $t$ and $u$ in a given context are identical.

\begin{theorem}
\label{thm_displacements}
For any alphabet $\Sigma$, CFG $G$ over $\Sigma$, and $t,u\in\Sigma^*$: $\Displacement(t)=\Displacement(u)\rightarrow t\SynCon u$ 
\end{theorem}
\begin{proof}
Appendix \ref{app_proof_thm_displacements}.
\end{proof}

The conclusion of Theorem \ref{thm_displacements} states that the relation $\SynConApprox$ defined by equivalence of displacements refines  $\SynCon$: it can be used for lossless compression in the \textsc{CFGzip} pipeline (Section \ref{sec_method_sub_syn_con}).  

\begin{algorithm}[t]
  \caption{Computation of $\Displacement(t)$}
  \label{alg_disp}
  \small
  \begin{algorithmic}[1]
    \Require token $t\in\Sigma^*$, CFG $G=(V,\Sigma,R,S)$
    \Statex \hspace*{2.3em} stack adjacency relation $\mathcal{A_S}\subseteq V\times V$

    \Procedure{CD}{$t,\sigma^{(i)},\sigma^{(o)},A_\textit{prev}$}
        \State $Y\gets\emptyset$
        \If{$t=\epsilon$}
            \State $Y\gets\{(\sigma^{(i)},\sigma^{(o)})\}$
        \ElsIf{$\sigma^{(o)}=\epsilon$} \Comment{backtrack}
            \For{$A\to t_0\hspace{1mm}\beta\in R$}
                \If{$(A_\textit{prev},A)\in\mathcal{A_S}$ \textbf{or} $A_\textit{prev}=\epsilon$}
                    \State $Y\gets Y\hspace{0.5mm}\cup\hspace{0.5mm}$\Call{CD}{$t_{1:},\sigma^{(i)}A,\beta,A$}
                \EndIf
            \EndFor
        \Else
            \For{$\beta\in\delta_G(t_0,\sigma^{(o)}_0)$}
                \State $Y\gets Y\hspace{0.5mm}\cup\hspace{0.5mm}$\Call{CD}{$t_{1:},\sigma^{(i)},\beta\sigma^{(o)}_{1:},\sigma^{(o)}_0$}
            \EndFor
        \EndIf
        \State \Return $Y$        
    \EndProcedure
    \State
    \State $\Delta_G(t)\gets$ \Call{CD}{$t,\epsilon,\epsilon,\epsilon$}
  \end{algorithmic}
\end{algorithm}

\subsubsection{Displacement Computation}
\label{sec_method_sub_approx_sub_disp_comp}

\begin{figure}[t]
\begin{subfigure}{.45\linewidth}
  \centering
  \includegraphics[width=.35\linewidth]{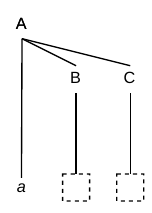}
    \caption{\textit{\textbf{\underline{a}} b c x y z} \\ $A\to a\hspace{1mm}B\hspace{1mm}C$ \\ $\sigma^{(i)}=(A)$ \\ $\sigma^{(o)}=(BC)$}
  \label{fig_stack_backtrack_1}
\end{subfigure}%
\begin{subfigure}{.45\linewidth}
  \centering
  \includegraphics[width=.35\linewidth]{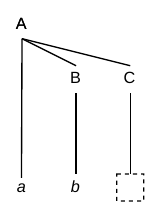}
  \caption{\textit{a \textbf{\underline{b}} c x y z} \\ $B\to b$ \\ $\sigma^{(i)}=(A)$ \\ $\sigma^{(o)}=(C)$}
  \label{fig_stack_backtrack_2}
\end{subfigure}
\\
\begin{subfigure}{.45\linewidth}
  \centering
  \includegraphics[width=.35\linewidth]{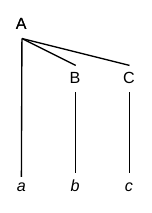}
  \caption{\textit{a b \textbf{\underline{c}} x y z} \\$C\to c$ \\ $\sigma^{(i)}=(A)$ \\ $\sigma^{(o)}=()$}
  \label{fig_stack_backtrack_3}
\end{subfigure}
\begin{subfigure}{.45\linewidth}
  \centering
  \includegraphics[width=.55\linewidth]{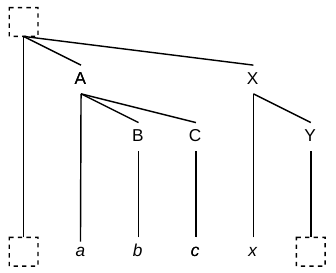}
  \caption{\textit{a b c \textbf{\underline{x}} y z} \\ $X\to x\hspace{1mm}Y$ \\ $\sigma^{(i)}=(AX)$ \\ $\sigma^{(o)}=(Y)$}
  \label{fig_stack_backtrack_4}
\end{subfigure}
\\
\begin{subfigure}{.45\linewidth}
  \centering
  \includegraphics[width=.9\linewidth]{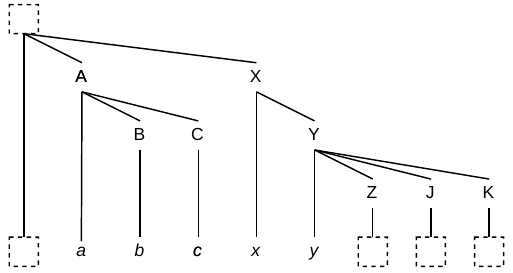}
    \caption{\textit{a b c x \textbf{\underline{y}} z} \\ $Y\to y\hspace{1mm}Z\hspace{1mm}J\hspace{1mm}K$ \\ $\sigma^{(i)}=(AX)$ \\ $\sigma^{(o)}=(ZJK)$}
  \label{fig_stack_backtrack_5}
\end{subfigure}%
\begin{subfigure}{.45\linewidth}
  \centering
  \includegraphics[width=.9\linewidth]{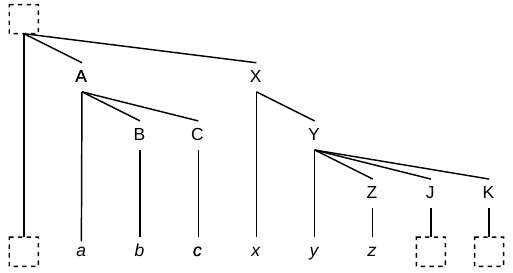}
      \caption{\textit{a b c x y \textbf{\underline{z}}} \\$Z\to z$ \\ $\sigma^{(i)}=(AX)$ \\ $\sigma^{(o)}=(JK)$}
  \label{fig_stack_backtrack_6}
\end{subfigure}%
\caption{Illustration of Algorithm \ref{alg_disp} applied to the token $t=\textit{abcxyz}$. Each sub-figure (\ref{fig_stack_backtrack_1}-\ref{fig_stack_backtrack_6}) illustrates the corresponding parse tree at that step: the caption indicates the current position/character $t_0$ (underlined/bolded), the rule $A\to t_0\hspace{1mm}\beta$ in $G$ that is applied (i.e. transition $\beta\in\delta_G(t_0,A)$), and the resulting input/output stacks $\sigma^{(i)}$/$\sigma^{(o)}$ after applying the transition.}
\label{fig_stack_backtrack}
\end{figure}

For each token $t\in T$, we compute $\Displacement(t)$ as in Algorithm \ref{alg_disp} (see the example given in Figure \ref{fig_stack_backtrack}). 


Given an input token $t$, we initialize an input queue $\sigma^{(i)}$ and an output stack $\sigma^{(o)}$. We first non-deterministically select a production $A\to t_0\hspace{1mm}\beta$ that produces the first character (byte) $t_0$ of $t$, and push $A$ and $\beta$ to $\sigma^{(i)}$ and $\sigma^{(o)}$, respectively (Figure \ref{fig_stack_backtrack_1}). We then proceed one character at a time along $t$, non-deterministically popping the top symbol $A$ from $\sigma^{(o)}$, and then pushing $\beta\in\delta_G(c,A)$ back to $\sigma^{(o)}$, where $c$ denotes the current character of $t$ (Algorithm \ref{alg_disp}, lines 9-11; Figures \ref{fig_stack_backtrack_2}-\ref{fig_stack_backtrack_3})\textemdash recall that $\beta\in\delta_G(c,A)$ corresponds to a production $A\to c\hspace{1mm}\beta$ in $G$ (see the discussion in Section \ref{sec_background_sub_cfgs}).

If the end of $t$ is reached, we return $(\sigma^{(i)},\sigma^{(o)})$, which indicates that $t$ transitions the stack $\sigma^{(i)}$ to $\sigma^{(o)}$ under $\delta_G$ (Algorithm \ref{alg_disp}, lines 3-4). However, if $\sigma^{(o)}$ is emptied before the end of $t$, we perform a \textit{stack backtrack}: non-deterministically select a production $A\to c\hspace{1mm}\beta$, enqueue $A$ to $\sigma^{(i)}$, and push $\beta$ to $\sigma^{(o)}$ (Algorithm \ref{alg_disp}, lines 5-8; Figures \ref{fig_stack_backtrack_3}-\ref{fig_stack_backtrack_4}).

As Algorithm \ref{alg_disp} is non-deterministic, we found in practice that the naive stack backtrack was intractable: the blowup in possible input sequences caused by the greater number of non-terminals in larger CFGs resulted in out-of-memory errors. 

We therefore restrict the stack backtrack operation using the pre-computed stack-adjacency relation $\mathcal{A_S}\subseteq V\times V$: given non-terminals $A,B\in V$, $(A,B)\in\mathcal{A_S}$ precisely when there is some string $w\in L(G)$ such that $B$ is popped off of the stack immediately after $A$ during some valid parse of $w$ with $\delta_G$ (see Appendix \ref{app_stack_adjacency} for a description of the stack-adjacency computation algorithm).

\begin{table*}[t]
\centering
\begin{tabular}{lc cc cc cc cc}
\toprule
& & \multicolumn{2}{c}{\textbf{JSON}} & \multicolumn{2}{c}{\textbf{XML}} & \multicolumn{2}{c}{\textbf{C++}} & \multicolumn{2}{c}{\textbf{Bython}} \\
\cmidrule(lr){3-4} \cmidrule(lr){5-6} \cmidrule(lr){7-8} \cmidrule(lr){9-10}
\textbf{Model} & \textbf{Vocab Size} & Time & Classes & Time & Classes & Time & Classes & Time & Classes \\
\midrule
Llama-3B & 128256 & 0.28 & 255.4 & 0.54 & 1591 & 1.29 & 3095 & 0.63 & 2011 \\
\midrule
Qwen-4B & 151936 & 0.33 & 257.8 & 0.64 & 1572 & 1.56 & 3057 & 0.71 & 1999 \\
\midrule
GPT-20B & 201088 & 0.49 & 253.4 & 0.96 & 1517 & 1.80 & 2815 & 1.08 & 1783 \\
\bottomrule
\end{tabular}
\caption{Wall-clock equivalence-class pre-computation time (seconds) and number of constructed equivalence classes per task for each model (vocabulary sizes given in number of tokens). JSON required a unique CFG for each of the 100 schemas: for this task, we report the mean across schema grammars of the pre-computation times and numbers of equivalence classes.}
\label{table_precompute}
\end{table*}

When computing $\Displacement(t)$ in Algorithm \ref{alg_disp}, we track the stack symbol $A_\textit{prev}$: this is either the symbol most recently popped from the output stack $\sigma^{(o)}$, or\textemdash if the last operation was a stack backtrack\textemdash the symbol most recently enqueued to $\sigma^{(i)}$. When a stack backtrack occurs, we may safely restrict the search space to those symbols $A$ such that $A$ is stack-adjacent to $A_\textit{prev}$\textemdash i.e. $(A_\textit{prev},A)\in\mathcal{A_S}$ (Algorithm \ref{alg_disp}, line 7).



\section{Experiments}
\label{sec_experiments}

To assess the effectiveness of \textsc{CFGzip}, we evaluated Llama-3.2-3B-Instruct\footnote{\url{https://huggingface.co/meta-llama/Llama-3.2-3B-Instruct}} \citep[Llama-3B;][]{grattafiori2024llama}, Qwen3-4B-Instruct\footnote{\url{https://huggingface.co/Qwen/Qwen3-4B-Instruct-2507}} \citep[Qwen-4B;][]{yang2025qwen3}, and gpt-oss-20b\footnote{\url{https://huggingface.co/openai/gpt-oss-20b}} \cite[GPT-20B;][]{agarwal2025gpt} on a series of structured-generation tasks (see Section \ref{sec_experiments_sub_tasks}). 

As a constrained baseline, we adopted a SoTA grammar engine, XGrammar2, using task-specific CFGs (see Section \ref{sec_experiments_sub_setup}); to evaluate our approach, we coupled \textsc{CFGzip} with XGrammar2. We compared both methods to an unconstrained baseline (i.e.\ the vanilla model).

\subsection{Tasks}
\label{sec_experiments_sub_tasks}

For evaluation, we selected four structured-generation tasks. In addition to compute time, we evaluated the models on each benchmark across two binary dimensions: (i) syntactic correctness, i.e.\ syntactic well-formedness of the model's output; and (ii) functional correctness, i.e.\ whether the output behaves as specified by the query.


\paragraph{JSON Schema:} The json-mode-eval\footnote{\url{https://huggingface.co/datasets/NousResearch/json-mode-eval}} dataset consists of 100 queries along with schemas specifying the required key strings and value data types of the target JSON objects. We measured syntactic correctness by executing the generated output and checking for syntax errors, and functional correctness by recursively checking key-value equivalence between the output and reference objects. 

\paragraph{XML:} We restricted the StructEval structured-generation benchmark \cite{yang2026structeval} to the 200 instances requiring XML generation. We measured syntactic correctness by parsing the output using the Python XML module\footnote{\url{https://docs.python.org/3/library/xml.html}}, and functional correctness by verifying that all dot-path rules specified by the instance were satisfied by the generated output code.

\paragraph{C++:} We restricted the HumanEval-X \cite{zheng2023codegeex} coding benchmark to the 164 C++ instances. We measured syntactic correctness by compiling the output, and calculated functional correctness by checking the output against all corresponding tests for each respective instance. 

\paragraph{Bython:} As the name suggests, Bython\footnote{\url{https://github.com/mathialo/bython}} is simply Python, but uses braces and semicolons to indicate scope and delimit statements (respectively), in place of meaningful whitespace. For our experiments, we translated all code in the 164 instances of the HumanEval-X Python split to Bython, and added a description of Bython to the system prompt. To evaluate syntactic correctness, we compiled the generated Bython code to Python. Functional correctness was measured as in the C++ task: namely, by checking if the output passed all tests for that instance.

\begin{table*}[ht]
\centering
\begin{tabular}{ll cc cc cc cc}
\toprule
& & \multicolumn{2}{c}{\textbf{JSON}} & \multicolumn{2}{c}{\textbf{XML}} & \multicolumn{2}{c}{\textbf{C++}} & \multicolumn{2}{c}{\textbf{Bython}} \\
\cmidrule(lr){3-4} \cmidrule(lr){5-6} \cmidrule(lr){7-8} \cmidrule(lr){9-10}
\textbf{Model} & \textbf{Constraint} & Syn. & Fun. & Syn. & Fun. & Syn. & Fun. & Syn. & Fun. \\
\midrule
\multirow{2}{*}{Llama-3B} & Unconstrained & 83.9 & 13.9 & 79.9 & 7.1 & 73.5 & 30.5 & 6.7 & 2.3 \\
 & Constrained & \textbf{99.8} & \textbf{73.7} & \textbf{85.1} & \textbf{16.3} & \textbf{75.0} & \textbf{32.7} & \textbf{52.1} & \textbf{9.2} \\
\midrule
\multirow{2}{*}{Qwen-4B} & Unconstrained & \textbf{100.0} & \textbf{76.5} & \textbf{83.9} & 61.1 & 88.9 & 70.6 & 2.2 & 1.0 \\
 & Constrained & 96.9 & 74.7 & 83.2 & \textbf{61.4} & \textbf{90.4} & \textbf{72.7} & \textbf{43.3} & \textbf{18.9} \\
\midrule
\multirow{2}{*}{GPT-20B} & Unconstrained & \textbf{100.0} & 80.2 & \textbf{82.1} & \textbf{60.1} & 48.4 & \textbf{41.2} & 13.0 & 10.8 \\
 & Constrained & \textbf{100.0} & \textbf{81.2} & 81.3 & 59.0 & \textbf{49.6} & 40.2 & \textbf{72.0} & \textbf{46.9} \\
\bottomrule
\end{tabular}
\caption{Constrained/unconstrained syntactic and functional correctness for each model and task. The best results for each model, task, and correctness type are indicated in bold.}
\label{table_correctness}
\end{table*}

\begin{table*}[ht]
\centering
\scalebox{0.93}{\begin{tabular}{llrrrr|rrr|rrr}
\toprule
& & & \multicolumn{3}{c}{\textbf{BS=1}} & \multicolumn{3}{c}{\textbf{BS=8}} & \multicolumn{3}{c}{\textbf{BS=16}} \\
\cmidrule(lr){4-6} \cmidrule(lr){7-9} \cmidrule(lr){10-12}
\textbf{Model} & \textbf{Constraint} & \multicolumn{1}{|c}{JSON} & \multicolumn{1}{|c}{XML} & \multicolumn{1}{c}{C++} & \multicolumn{1}{c}{By.} & \multicolumn{1}{|c}{XML} & \multicolumn{1}{c}{C++} & \multicolumn{1}{c}{By.} & \multicolumn{1}{|c}{XML} & \multicolumn{1}{c}{C++} & \multicolumn{1}{c}{By.} \\
\midrule
\multirow{3}{*}{Llama-3B} & XGrammar2 &
 \multicolumn{1}{|c}{3.60} &
 \multicolumn{1}{|r}{4.18} & 2.14 & 12.55 &
 1.51 & 1.27 & 44.51 &
 0.77 & 0.85 & \textemdash \\
 & +\textsc{CFGzip} &
 \multicolumn{1}{|c}{3.52} &
 \multicolumn{1}{|r}{3.71} & 1.47 & 1.76 &
 0.78 & 0.59 & 0.59 &
 0.43 & 0.35 & 0.50 \\
 & Unconstrained &
\multicolumn{1}{|c}{3.42} &
\multicolumn{1}{|r}{3.67} & 1.42 & 1.39 &
0.68 & 0.46 & 0.25 &
0.40 & 0.30 & 0.24 \\
\midrule
\multirow{3}{*}{Qwen-4B} & XGrammar2 &
 \multicolumn{1}{|c}{1.75} &
 \multicolumn{1}{|r}{6.22} & 5.14 & 19.50 &
 1.29 & 2.62 & 97.07 &
 0.97 & 4.21 & \textemdash \\
 & +\textsc{CFGzip} &
 \multicolumn{1}{|c}{1.53} &
 \multicolumn{1}{|r}{5.80} & 2.55 & 2.72 &
 0.84 & 0.63 & 1.40 &
 0.66 & 0.54 & 0.82 \\
 & Unconstrained &
\multicolumn{1}{|c}{1.52} &
\multicolumn{1}{|r}{5.76} & 2.48 & 2.34 &
0.78 & 0.53 & 0.45 &
0.62 & 0.47 & 0.45 \\
\midrule
\multirow{3}{*}{GPT-20B} & XGrammar2 &
 \multicolumn{1}{|c}{1.99} &
 \multicolumn{1}{|r}{6.87} & 4.59 & 19.52 &
 2.50 & 6.69 & \textemdash &
 2.08 & 10.49 & \textemdash \\
 & +\textsc{CFGzip} &
 \multicolumn{1}{|c}{1.68} &
 \multicolumn{1}{|r}{6.23} & 2.41 & 3.52 &
 1.76 & 1.29 & 1.89 &
 1.60 & 1.56 & 2.11 \\
 & Unconstrained &
\multicolumn{1}{|c}{1.67} &
\multicolumn{1}{|r}{6.20} & 2.34 & 2.67 &
1.67 & 0.97 & 1.50 &
1.55 & 1.05 & 1.82 \\
\bottomrule
\end{tabular}}
\caption{Mean generation time per instance (in seconds)\textemdash including the LLM overhead\textemdash for each batch size, model, constraint method, and task. A dash (``\textemdash '') indicates that the evaluation for that cell timed out after 24 hours.}
\label{table_inference_time}
\end{table*}

\paragraph{} While CFG engines provide guarantees of syntactically correct generation, we do not expect them to yield large improvements for tasks requiring common, ``high-resource'' context-free languages (CFLs) such as JSON, XML, and C++: they are ubiquitous in LLM training data, and so it is likely that SoTA models are already highly capable of generating strings that conform syntactically to these CFLs. On the other hand, a rare and/or domain-specific CFL such as Bython is far more likely to confound an LLM, and so it is precisely for these ``low-resource'' CFLs that we expect CFG-constrained generation to be most useful.

\subsection{Setup}
\label{sec_experiments_sub_setup}

\paragraph{Hyperparameters and Evaluation.} All models/methods were evaluated on all tasks across five seeds with a temperature of 0.2 and a maximum of 500 generated tokens. For the XML, C++, and Bython tasks, we evaluated at batch sizes of 1, 8, and 16. As the JSON task requires a separate grammar for each request, it is not conducive to batched inference with the current iteration of \textsc{CFGZip}: we therefore only evaluated on this task using a batch size of 1. 

Inference was performed on a single NVIDIA H100 GPU, and the pre-computation and grammar engine operations were run 16-threaded on an Intel Xeon Gold 6430 CPU. Additional details of our evaluation setup are located in Appendix \ref{app_setup_details}.


\paragraph{Grammars.} For the JSON task, we used XGrammar2's built-in function for constructing CFGs from JSON schema specifications. We adapted an XML 1.1 specification\footnote{\url{https://www.liquid-technologies.com/Reference/Glossary/XML_EBNF1.1.html}} to the GBNF required by XGrammar2 for the XML task. For the C++ task, we used the C++ grammar from \citet{mundler2025constrained}. We constructed our Bython CFG from scratch: to the best of our knowledge, there is no existing Bython grammar specification.

\begin{figure*}[t]
\begin{center}
\includegraphics[width=1.0\textwidth]{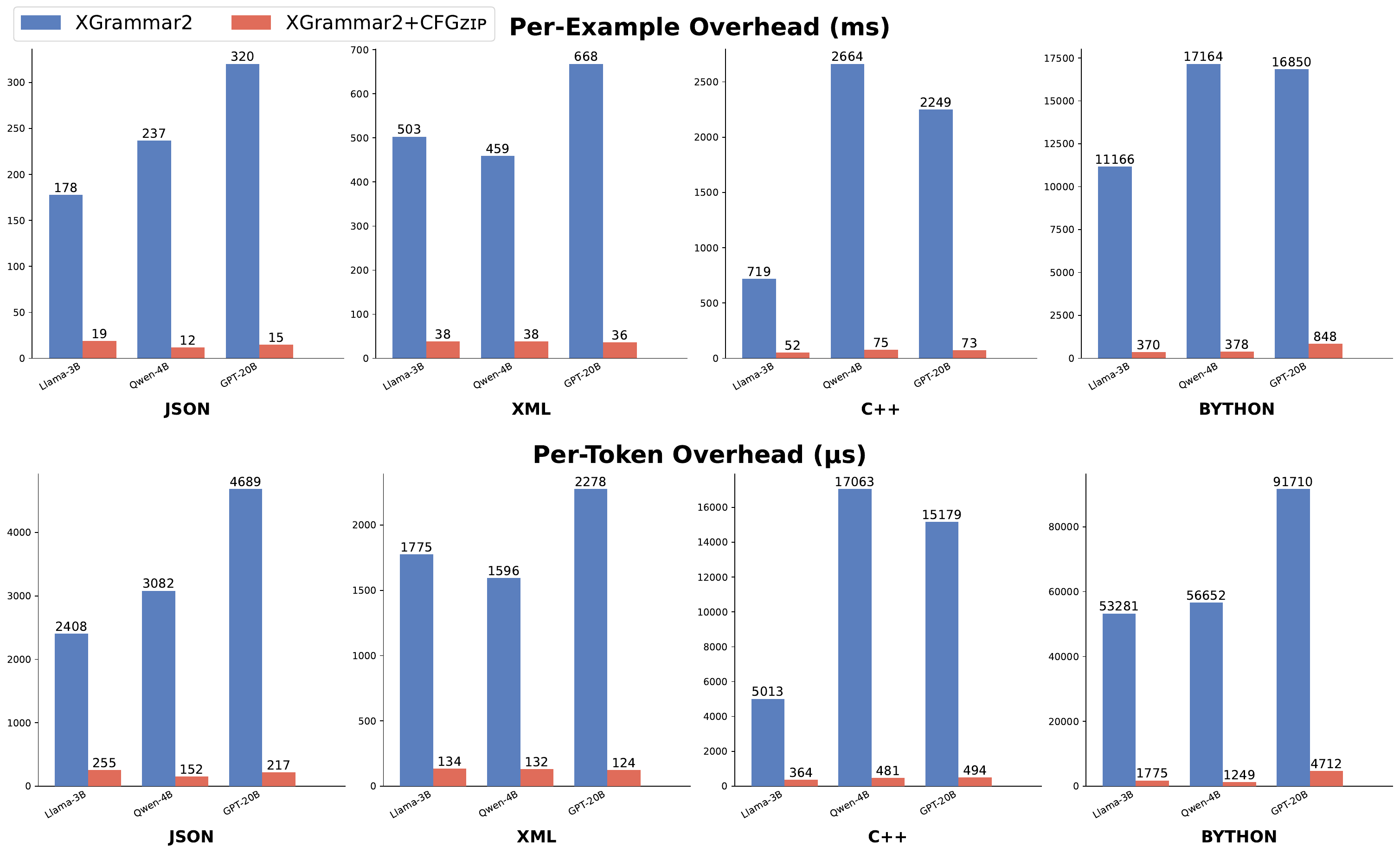}
\end{center}
\caption{Comparison of the mean per-example (top; milliseconds) and per-token (bottom; microseconds) unbatched overhead of vanilla XGrammar2 (blue) compared to XGrammar2 equipped with \textsc{CFGzip} (red).}
\label{fig_latency}
\end{figure*}

\paragraph{Pre-Computation.} Wall-clock times for the \textsc{CFGzip} pre-computation procedure (see Section \ref{sec_method_sub_syn_con}) and the number of equivalence classes for each tokenizer/CFG are reported in Table \ref{table_precompute}: \textsc{CFGzip} achieves compression ratios ranging from $\sim$40:1 (Llama-3B/C++) to $\sim$800:1 (GPT-20B/JSON).


\section{Results}
\label{sec_results}


Syntactic/functional correctness for each constrained/unconstrained model on each task are given in Table \ref{table_correctness}. We do not report separate scores for \textsc{CFGzip}, as it produces byte-identical outputs to those of the unmodified XGrammar2. 

Although constrained decoding improves syntactic correctness for most models on most tasks, in most cases it does not yield a 100\% syntactic correctness score. This behavior is to be expected, as the grammar engine only ensures that the generated sequence is a valid \textit{prefix} (see the discussion in Section \ref{sec_background_sub_xgrammar}): if the model's generation is cut off by the 500-token limit, the resulting partial sequence typically does not pass the syntactic correctness checks.

\paragraph{Constrained Decoding Improves Performance with Smaller Models and Unfamiliar CFLs.}
The across-the-board improvement in correctness for Llama-3B demonstrates the utility of constrained decoding for smaller\textemdash and therefore cheaper\textemdash LLMs: rather than using a larger model to improve performance, simply constraining the smaller LLM presents a viable, lower-cost alternative. On the other hand, the massive improvement in both syntactic and functional correctness for all three models on the Bython task shows that CFG-constrained generation is critical for using complex and unfamiliar context-free languages\textemdash even with 20-billion-parameter models such as GPT-20B. 




\paragraph{SoTA Grammar Engines Have Unsustainable Overhead with Complex CFGs.} XGrammar2 increases unbatched (i.e.\ batch size of 1) inference wall-clock time over the base (i.e.\ unconstrained) model by factors of $\sim$2x and $\sim$9x for C++ and Bython, respectively (see Table \ref{table_inference_time}). This overhead is exacerbated at scale: during batched inference, the grammar engine must wait for the slowest parser state in the batch to compute its token mask before proceeding to the next generation step. 

As a result, \textbf{XGrammar2 inference time increases to $\sim$3-7x (C++) and $\sim$100-200x (Bython)} that of the base model with batch size of 8. At a batch size of 16, \textbf{all three XGrammar2-constrained models timed out after 24 hours on the Bython task}. In production environments necessitating large batch sizes, constrained decoding with complex CFGs is infeasible\textemdash even when using SoTA grammar engines such as XGrammar2.







\paragraph{\textsc{CFGzip} Makes Constrained Decoding with Complex CFGs Viable at Scale.} By directly compressing the token space, \textsc{CFGzip} nearly halves the total unbatched inference time of XGrammar2 on C++, and delivers batched constrained decoding speedups for XML and C++ by factors of $\sim$1.3-2x and $\sim$2-8x, respectively. On the Bython task, our method cuts the unbatched overhead of XGrammar2 by up to 45x (see Figure \ref{fig_latency}): \textbf{unbatched constrained generation time decreases from $\sim$7-9x that of the base model with vanilla XGrammar2 to $\sim$1.3x with \textsc{CFGzip}}.



At a batch size of 8, constrained generation time with \textsc{CFGzip} is only $\sim$1.2-3x that of the base model on Bython\textemdash including for GPT-20B, where XGrammar2 alone times out. For Llama-3B and Qwen-4B, this represents a \textbf{$\sim$75x speedup over vanilla XGrammar2} at a batch size of 8. As batch size increases to 16 on Bython, \textbf{constrained generation time with \textsc{CFGzip} remains at $\sim$1.2-2x that of the base model, while vanilla XGrammar2 times out after 24 hours}. With \textsc{CFGzip}, constrained decoding workloads that were previously impossible at scale can be performed with near-negligible overhead.

\section{Related Work}
\label{sec_related_work}

There are several existing constrained decoding engines for structured generation. XGrammar2 achieves faster CFG compilation and mask generation than its predecessor, XGrammar \cite{dong2025xgrammar}, by swapping the PDA/stack-based parser of the latter for an Earley parser, and incorporating additional optimizations such as just-in-time grammar compilation. The llguidance engine\footnote{\url{https://github.com/guidance-ai/llguidance}} similarly adopts an Earley parser that is highly optimized for use with JSON schema grammars. 

The transformers-cfg engine \cite{geng-etal-2023-grammar} employs an incremental parser \cite{angelov-2009-incremental}, which is similar to the PDA parser used by the built-in CFG engine of the llama.cpp\footnote{\url{https://github.com/ggml-org/llama.cpp}} inference accelerator. 

GreatGramma \cite{park2025flexible} also employs a PDA parser, and reduces overhead by pre-computing possible next tokens for certain stack prefixes. However, this engine is restricted to deterministic CFGs and does not ensure syntactic correctness: in some contexts, it fails to mask invalid tokens. SynCode \cite{ugare2025syncode} pre-computes possible next tokens for terminal finite-state machine (FSM) states, but is restricted to LR-parsable CFGs and can also fail to mask invalid tokens.

Outlines \cite{willard2023efficient} and lm-format-enforcer\footnote{\url{https://github.com/noamgat/lm-format-enforcer}} use FSM- and tree-based parsers (respectively) for constrained decoding parameterized by regular expressions and JSON schemas, but are not compatible with CFGs in general.

\textsc{CFGzip} is orthogonal and complementary to these methods: it is an algorithm for token space compression that can in theory be combined with any token-masking approach for CFG-constrained decoding. To the best of our knowledge, \textsc{CFGzip} represents the first tool for compressing the token search space of a grammar engine through CFG-based equivalence classing.


\section{Conclusion}
\label{sec_conclusion}

We introduced \textsc{CFGzip}, an LLM token vocabulary compression technique for inference-time speedup during CFG-constrained decoding. Our method achieves a $\sim$9-45x reduction in unbatched overhead across model/token vocabulary sizes and CFGs. The speedup provided by \textsc{CFGzip} scales with batch size, yielding up to $\sim$75x faster batched constrained generation. As a result, tasks on which XGrammar2 alone times out can be accomplished using \textsc{CFGzip} with only $\sim$1.2-2x latency relative to the base model.

Our method unlocks previously infeasible performance from\textemdash and, therefore, applications of\textemdash CFG engines: although constrained decoding can more than quadruple functional correctness on a complex and unfamiliar context-free language such as Bython, even a SoTA grammar engine cannot handle the complexity of that language without introducing intractable overhead. \textsc{CFGzip} shrinks this overhead to near-negligible levels\textemdash including at scale\textemdash allowing the practical use of CFG-constrained decoding in the settings where it is most effective. 

\section*{Acknowledgments}

We gratefully acknowledge the stimulating research environment of the GRK 2853/1 ``Neuroexplicit Models of Language, Vision, and Action'', funded by the Deutsche Forschungsgemeinschaft (DFG; German Research Foundation) under project number 471607914.

\section*{Limitations}

For smaller, simpler CFGs (e.g.\ JSON schema grammars), the non-negligible offline pre-computation time for \textsc{CFGzip} is longer than the reduction in overhead it provides (see Table \ref{table_precompute} and Figure \ref{fig_latency}): this precludes the utility of \textsc{CFGzip} for dynamic, one-off use cases with lightweight CFGs. Our method is intended instead for contexts requiring repeated use of the same grammar and/or larger, more complex CFG, where the inference-time speedup provided by \textsc{CFGzip} ameliorates its pre-computation overhead. 

Although \textsc{CFGzip} is theoretically engine-agnostic, we only evaluated our method with XGrammar2 in this paper, due to the below-listed incompatibilities of the other grammar engines discussed in Section \ref{sec_related_work}. 

Outlines\footnote{Outlines implements a method to compile arbitrary CFGs, but calling this method currently raises a NotImplementedError.} and lm-format-enforcer are only compatible with regular expressions and JSON schemas, and so cannot be used with the XML, C++, and Bython grammars of Section \ref{sec_experiments}. Similarly, GreatGramma and SynCode are not compatible with all CFGs\textemdash and do not guarantee syntactic correctness\textemdash so we did not evaluate our method with those engines. 

The transformers-cfg package does not currently support the gpt-oss tokenizer, and so could not be used to evaluate that model. More importantly, transformers-cfg cannot parse CFGs with left-recursive productions, making it incompatible with our C++ and Bython grammars.

We found that llguidance incorrectly rejected grammatical C++ and Bython strings that were accepted by XGrammar2: we suspect this is due to the engine's dynamic lexer prematurely rejecting valid paths under certain conditions. As it did not function correctly for half of our tasks, we did not evaluate \textsc{CFGzip} with this engine. 

The llama.cpp engine was incompatible with \textsc{CFGzip} for practical reasons: \textsc{CFGzip} is implemented in Python/Rust, while llama.cpp is implemented in C++ and does not expose Python bindings. 


\bibliography{anthology,custom}

\appendix

\section{AI Use}
\label{app_ai_use}

The initial \textsc{CFGzip} code was implemented in Python by hand. We used AI coding assistants to help port the research repository to a pip-installable module, including converting the implementation of the displacement-construction algorithm (see Section \ref{sec_method_sub_approx_sub_disp_comp}) component from pure Python to a Rust backend to decrease pre-computation overhead.

\section{Proof of Theorem \ref{thm_token_seq_equiv}}
\label{app_proof_thm_token_seq_equiv}

First, given a string $w\in\Sigma^*$, define the left quotient $w\backslash L$ of $L$ by $w$ as the set of all $z\in\Sigma^*$ such that $wz\in L$: $w\backslash L=\{z\in\Sigma^*\hspace{1mm}|\hspace{1mm}wz\in L\}$. It is clear that the definition of syntactic congruence $t\equiv_L u$ in Definition \ref{def_syn_congruence} is equivalent to the condition that for all $w\in\Sigma^*$: $wt\backslash L=wu\backslash L$.

\begin{equation}
\label{eq_syn_con_quotient}
    t\equiv_L u\leftrightarrow\forall w\in\Sigma^*\colon wt\backslash L=wu\backslash L
\end{equation}

We now prove that the quotient map $\Sigma^*\to\Sigma^*/\equiv_L$ is a monoid homomorphism.

\begin{lemma}
\label{lem_syn_con_cat}
For any $w,w',z,z'\in\Sigma^*$, $w\equiv_Lw'$ and $z\equiv_Lz'$ implies that $wz\equiv_Lw'z'$.
\end{lemma}
\begin{proof}
    By assumption and Definition \ref{def_syn_congruence}, we have the following equivalences for all $x,y\in\Sigma^*$:
    
    \begin{enumerate}[(i)]
        \item $xwy\in L\leftrightarrow xw'y\in L$
        \item $xzy\in L\leftrightarrow xz'y\in L$
    \end{enumerate}
    
    By (i), we have $xwzy\in L\leftrightarrow xw'zy\in L$, i.e. $wz\equiv_Lw'z$. By (ii), we have $xw'zy\in L\leftrightarrow xw'z'y\in L$, i.e. $w'z\equiv_Lw'z'$. Chaining equivalences, we have: $wz\equiv_Lw'z\equiv_Lw'z'$.
\end{proof}

By Lemma \ref{lem_syn_con_cat}, $\tau\equiv_L\tau'$, and so by Equation \ref{eq_syn_con_quotient}, we have $\tau\backslash L=\tau'\backslash L$.

For any $a\in\Sigma^*$, the condition $a\in\textit{Pr}(L)$ is equivalent to the condition that $a\backslash L\neq\emptyset$. Therefore, $\tau\backslash L=\tau'\backslash L$ implies that $\tau\in\textit{Pr}(L)\leftrightarrow\tau'\in\textit{Pr}(L)$. 

This completes the proof of Theorem \ref{thm_token_seq_equiv}.

\section{Proof of Theorem \ref{thm_displacements}}
\label{app_proof_thm_displacements}

Define $\pi\colon \Sigma^*\times V^*\to\mathcal{P}(V^*)$ to be the function that takes a sequence of terminals $t$ (token) and a sequence of non-terminals $\alpha$ (stack) as input, and returns the set of stacks obtained by applying the transition function $\delta_G$ to $t$ given $\alpha$ (Equation \ref{eq_pi}).

\begin{equation}
\label{eq_pi}
    \pi(t,\alpha)=\begin{cases}
    \{\alpha\} &\text{if }t=\epsilon \\
    \emptyset &\text{if }\alpha=\epsilon \\
    \underset{\beta\in\delta_G(t_0,\alpha_0)}{\bigcup} \pi(t_{1:},\beta\alpha_{1:}) &\text{otw.}
    \end{cases}
\end{equation}

Now define $\pi_S\colon\Sigma^*\to\mathcal{P}(V^*)$ as the restriction of $\pi$ to the start symbol $S$: $\pi_S(t)=\pi(t,S)$. 

We now prove in Lemma \ref{lem_pi_quotient} that equivalence of $\pi_S$-sets implies equivalence of left quotients (see Equation \ref{eq_syn_con_quotient}).

\begin{lemma}
\label{lem_pi_quotient}
    For all $w,z\in\Sigma^*$: $\pi_S(w)=\pi_S(z)$ implies that $w\backslash L(G)=z\backslash L(G)$. 
\end{lemma}
\begin{proof}
    For all $x\in\Sigma^*$, $\pi_S(x)$ is by definition the result of non-deterministically applying $\delta_G$ to $x$, starting from $S$. By definition, $x\in L(G)$ if and only if $\delta_G$ halts on the empty stack. Therefore, we have $x\in L(G)\leftrightarrow\epsilon\in\pi_S(x)$.
    
    Applying the definition of the left quotient and the above equivalence, we have $x\backslash L(G)=\{y\hspace{1mm}|\hspace{1mm}xy\in L(G)\}=\{y\hspace{1mm}|\hspace{1mm}\epsilon\in\pi_S(xy)\}$. As $\pi_S(xy)=\bigcup_{\beta\in\pi_S(x)}\pi(y,\beta)$ by definition (Equation \ref{eq_pi}), we have $x\backslash L(G)=\{y\hspace{1mm}|\hspace{1mm}\epsilon\in\pi^*(x,y)\}$, where $\pi^*(x,y)=\bigcup_{\beta\in\pi_S(x)}\pi(y,\beta)$.

    By assumption, $\pi_S(w)=\pi_S(z)$, which implies that $\pi^*(w,\text{\textendash})=\pi^*(z,\text{\textendash})$ pointwise. By the above equivalence, this in turn implies that $w\backslash L(G)=z\backslash L(G)$.
\end{proof}

We now define the \textit{pre-displacement} $D_G(t)$ of a token $t$: $D_G(t)=\{(\alpha,\beta)\hspace{1mm}|\hspace{1mm}\alpha\in V^*\land\beta\in\pi(t,\alpha)\}$. The relation defined by equivalence of pre-displacements refines syntactic congruence (Lemma \ref{lem_pre_displacement}).

\begin{lemma}
\label{lem_pre_displacement}
    For all $t,u\in\Sigma^*$: $D_G(t)=D_G(u)$ implies that $t\SynCon u$.
\end{lemma}
\begin{proof}
    The assumption that $D_G(t)=D_G(u)$ implies by definition of $D_G(\text{\textendash})$ that $\pi(t,\gamma)=\pi(u,\gamma)$ for each $w\in\Sigma^*$ and each $\gamma\in\pi_S(w)$. This in turn implies that $\pi_S(wt)=\pi_S(wu)$, which by Lemma \ref{lem_pi_quotient} implies that $wt\backslash L(G)=wu\backslash L(G)$.

    We therefore have $wt\backslash L(G)=wu\backslash L(G)$ for all $w\in\Sigma^*$: this is equivalent to the definition of $t\SynCon u$ (see Equation \ref{eq_syn_con_quotient}).
\end{proof}

Given a token $t$ and a pair $(\alpha,\beta)\in D_G(t)$, at most $\textit{len}(t)$ symbols are consumed from $\alpha$ when parsing $t$. This is to say that there is some $1\leq n\leq\textit{len}(t)$ such that only the first $n$ symbols of $\alpha$ are consumed, and the remaining subsequence $\gamma$ of $\alpha$ is shared with $\beta$: $\alpha=\hat{\alpha}\gamma$ where $\hat{\alpha}=A_1\dots A_n\in V^*$, and $\beta=\hat{\beta}\gamma$ where $\hat{\beta}\in V^*$.


$D_G(t)$ is thus \textit{precisely} the set of pairs $(\hat{\alpha}\gamma,\hat{\beta}\gamma)$ such that $\hat{\alpha}$ is entirely consumed when parsing $t$, $\hat{\beta}\in\pi(t,\hat{\alpha})$, and $\gamma\in V^*$.



Now, define $V^{(0)}=V\times\{0\}$, $V^{(1)}=V\times\{1\}$, $f_i\colon V^*\to(V^{(i)})^*$ as $f_i(A_0\dots A_n)=(A_0,i)\dots(A_n,i)$, and $g\colon(V^{(0)}\cup V^{(1)})^*\to V^*$ as $g((A_0,i_0)\dots(A_n,i_n))=A_0\dots A_n$. Define $\pi'\colon\Sigma^*\times(V^{(0)}\cup V^{(1)})^*\to\mathcal{P}((V^{(0)}\cup V^{(1)})^*)$ analogously to $\pi$ in Equation \ref{eq_pi}, with the exception that $\pi'$ takes sequences of symbols of the form $(A,i)$ (for $i\in\{0,1\}$), and pushes to the stack sequences of symbols of the form $(B,1)$\textemdash the 0/1 are used to track whether the symbol was present in the initial input stack. 

\begin{equation}
\label{eq_pi_prime}
    \pi'(t,\alpha)=\begin{cases}
    \{\alpha\} &\text{if }t=\epsilon \\
    \emptyset &\text{if }\alpha=\epsilon \\
    \underset{\beta\in\delta_G(t_0,g(\alpha_0))}{\bigcup} \pi'(t_{1:},f_1(\beta)\alpha_{1:}) &\text{otw.}
    \end{cases}
\end{equation}

Define $D'_G(t)=\{(g(\alpha),g(\beta))\hspace{1mm}|\hspace{1mm}\alpha\in(V^{(0)})^*\land \beta\in\pi'(t,\alpha)\cap(V^{(1)})^*\}$. This is exactly $D_G(t)$ restricted to those pairs $(\alpha,\beta)$ such that $\alpha$ is entirely consumed when parsing $t$ to derive $\beta$. 

By the reverse of the above argument, we have that $D_G(t)=\{(\alpha\gamma,\beta\gamma)\hspace{1mm}|\hspace{1mm}(\alpha,\beta)\in D'_G(t)\land\gamma\in V^*\}$. Therefore, we have $D'_G(t)=D'_G(u)\rightarrow D_G(t)=D_G(u)$ for all $t,u\in\Sigma^*$. Chained with Lemma \ref{lem_pre_displacement}, we now have $D'_G(t)=D'_G(u)\rightarrow t\SynCon u$.

As defined in Algorithm \ref{alg_disp}, the displacement of a token $t$ is precisely $\Displacement(t)=\{(\alpha,\beta)\in D'_G(t)\hspace{1mm}|\hspace{1mm}\forall0< i<\textit{len}(\alpha):(\alpha_{i-1},\alpha_i)\in\mathcal{A_S}\}$, where $\mathcal{A_S}$ is the stack adjacency relation (see Appendix \ref{app_stack_adjacency}).

This only eliminates from $D'_G(t)$ those pairs $(\alpha,\beta)$ such that $\alpha$ contains a pair $(A_i,A_{i+1})$ of non-terminals such that $A_{i+1}$ cannot be popped off the stack immediately after $A_i$ during any parse starting from $S$. It therefore still holds that $\Displacement(t)=\Displacement(u)\rightarrow\pi_S(wt)=\pi_S(wu)$ for all $w\in\Sigma^*$, which by the argument in Lemma \ref{lem_pre_displacement} implies that $t\SynCon u$.

This completes the proof of Theorem \ref{thm_displacements}.




\section{Stack Adjacency}
\label{app_stack_adjacency}

Note that we require only a \textit{subset} of the stack adjacency relation $\mathcal{A_S}\subseteq V\times V$: in practice, we may limit $\mathcal{A_S}$ to those pairs $(X,Y)$ such that a stack backtrack can occur in Algorithm \ref{alg_disp} after $X$, as we are computing this relation solely for the purpose of restricting the stack backtrack operation. By definition, a stack backtrack only occurs when the stack $\sigma^{(o)}$ is empty: i.e. $X$ is the only symbol on the stack when it is popped, and no symbols are pushed back onto the stack\textemdash this corresponds to the application of a rule of the form $X\to c$ in $G$. 

\paragraph{Pre-computation.} We pre-compute $\mathcal{A_S}$ only once before constructing $\Displacement(t)$ for each token $t$. We first compute the directed graph $\mathcal{H}=(\mathcal{V},\mathcal{E})$, where $\mathcal{V}=V$ is the set of non-terminals of the CFG $G$, and there is an edge $A\rightarrow B\in\mathcal{E}$ for each production $A\rightarrow a\hspace{1mm}\beta\hspace{1mm}B$ in $G$ ($\beta\in V^*$)\textemdash i.e. where $B$ is the right-most non-terminal of a rule headed by $A$. 

\begin{figure}[t]
\begin{center}
\includegraphics[width=1.0\linewidth]{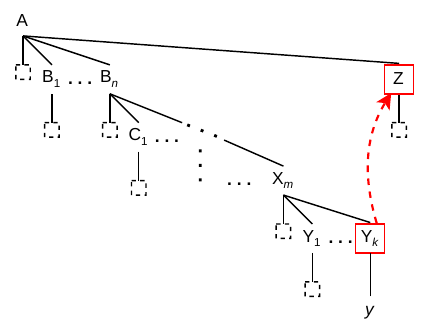}
\end{center}
\caption{Illustration of the stack adjacency sub-relation $\mathcal{A_S}$. Here, $Y_k$ is stack-adjacent to $Z$\textemdash i.e. $(Y_k,Z)\in\mathcal{A_S}$: there is some possible parse in which $Y_k$ is popped off the stack immediately before $Z$.}
\label{fig_poss_bigrams}
\end{figure}

As illustrated in Figure \ref{fig_poss_bigrams}, we calculate $\mathcal{A_S}$ from $\mathcal{H}$ as follows. For each non-terminal $Z$ and each production $A\rightarrow a\hspace{1mm}B_1\dots B_n\hspace{1mm}Z\hspace{1mm}\beta\in R$ ($\beta\in V^*$), we perform a breadth-first search on $\mathcal{H}$ starting from $B_n$ to find all $Y_k$ such that there is a path from $B_n$ to $Y_k$ in $\mathcal{H}$ and there is a unary production $Y_k\rightarrow y$ in $G$: each such $Y_k$ is popped from the stack immediately before $Z$ during some valid parse\textemdash and can result in a stack backtrack in Algorithm \ref{alg_disp}\textemdash and so we add $(Y_k,Z)$ to $\mathcal{A_S}$.

\section{Experimental Setup: Additional Details}
\label{app_setup_details}

For the Bython task of Section \ref{sec_experiments}, several instances took multiple hours to generate with XGrammar2: for this reason, we placed a 15-minute generation timeout window on all models for this task. This did not impact the correctness scores in Table \ref{table_correctness}: instances for which the XGrammar2-constrained model timed out were invariably pathological examples in which the model was caught in a loop of generating incoherent text.

Roughly 5\% of all examples timed out with XGrammar2 across all five seeds (see Table \ref{table_timeout}), versus only 0.2\% of the examples for the \textsc{CFGzip}-equipped GPT model\textemdash and 0.0\% for all other \textsc{CFGzip}-equipped models\textemdash corresponding to two examples out of 164 across five seeds (i.e. 2/820).

We treated these timed-out examples as outliers and did \textit{not} include them when calculating the overhead times reported in Figure \ref{fig_latency} or the unbatched inference times in Table \ref{table_inference_time}, and we entirely removed those instances from the dataset for batched evaluation: the speedup afforded by \textsc{CFGzip} is not attributable to those pathological examples. 

\begin{table}[t]
\centering
\scalebox{0.87}{\begin{tabular}{l|ccc}
 & \textbf{Llama-3B} & \textbf{Qwen-4B} & \textbf{GPT-20B} \\
\midrule
Unconstrained & 0.0 & 0.0 & 0.0 \\
XGrammar2 & 5.1 & 5.2 & 4.9 \\
+\textsc{CFGzip} & 0.0 & 0.0 & 0.2 \\
\end{tabular}}
\caption{Percentage of Bython task instances across all five seeds that timed out after 15 minutes, by model and constraint type.}
\label{table_timeout}
\end{table}

\end{document}